%% file: root.tex
\pdfoutput=1

\documentclass[letterpaper, 10 pt, conference]{ieeeconf} 

\IEEEoverridecommandlockouts                              

\usepackage{eso-pic}
\AddToShipoutPicture*{\small \sffamily\raisebox{1.2cm}{\hspace{1.9cm}Copyright \copyright \ 2024 IEEE.}}

\input{macros_ieeeconf}

\title{\textbf{Robust Stochastic Shortest-Path Planning via Risk-Sensitive Incremental Sampling}\textsuperscript{*}
\thanks{\textsuperscript{*}The material in this work received support in part from the Office of Naval Research (Grant No. N000141712622) and from a seed grant awarded by Northrop Grumman Corporation.}
\thanks{\textsuperscript{1}The authors are with the Department of Electrical \& Computer Engineering and the Institute for Systems Research, University of Maryland, College Park, MD.
Emails: \{\texttt{enwerem, enoorani, baras}\}@umd.edu.}
\thanks{\textsuperscript{2}B. M. Sadler is with the University of Texas at Austin, TX. Email: {brian.sadler}@ieee.org.
}
}

\author{{Clinton Enwerem\textsuperscript{1}}
\and
{Erfaun Noorani\textsuperscript{1}}
\and
{John S. Baras\textsuperscript{1}}
\and
{Brian M. Sadler\textsuperscript{2}}
}

\begin{document}

\maketitle
\thispagestyle{empty}
\pagestyle{empty}

\begin{abstract}
With the pervasiveness of Stochastic Shortest-Path (SSP) problems in high-risk industries, such as last-mile autonomous delivery and supply chain management, robust planning algorithms are crucial for ensuring successful task completion while mitigating hazardous outcomes. Mainstream chance-constrained incremental sampling techniques for solving SSP problems tend to be overly conservative and typically do not consider the likelihood of undesirable tail events. We propose an alternative risk-aware approach inspired by the asymptotically-optimal Rapidly-Exploring Random Trees (RRT*) planning algorithm, which selects nodes along path segments with minimal Conditional Value-at-Risk (CVaR). Our motivation rests on the step-wise coherence of the CVaR risk measure and the optimal substructure of the SSP problem. Thus, optimizing with respect to the CVaR at each sampling iteration necessarily leads to an optimal path in the limit of the sample size. We validate our approach via numerical path planning experiments in a two-dimensional grid world with obstacles and stochastic path-segment lengths. Our simulation results show that incorporating risk into the tree growth process yields paths with lengths that are significantly less sensitive to variations in the noise parameter, or equivalently, paths that are more robust to environmental uncertainty. Algorithmic analyses reveal similar query time and memory space complexity to the baseline RRT* procedure, with only a marginal increase in processing time. This increase is offset by significantly lower noise sensitivity and reduced planner failure rates.
\end{abstract}

\section{Introduction}\label{sec:intro}
Many practical robot motion planning tasks can be recast as shortest-path optimization problems, ranging from problems in autonomous drone delivery \cite{huang2020reliable} to search and rescue \cite{lin2009uav}. In the basic setting, the {Stochastic Shortest-Path} (SSP) planning problem \revadd{involves finding} a path of minimum \textit{expected} length \revadd{that connects} two obstacle-free robot configurations in a planning environment with obstacles and some notion of uncertainty. Such uncertainty may arise from the robot's sensing \cite{blackmore_probabilistic_2006, thrun_probabilistic_2005}, decision-making, or actuation modules \cite{tao_path_2022,williams_information_2017}. Furthermore, many established techniques for finding shortest paths --- such as graph search via Dijkstra's or the A* algorithm or sampling-based algorithms like the Rapidly-Exploring Random Trees (RRT) algorithm \cite{lavalle2001rapidly} and probabilistic roadmaps \cite{kavraki1996probabilistic} --- are fundamentally non-robust. The foregoing challenges thus provide motivation for an inquiry into the design of robust shortest-path planning algorithms that consider some form of uncertainty in the planning procedure and return adequately-adjusted plans.

In the \revadd{deterministic} path planning setting, the shortest-path problem can be solved efficiently, for instance, via graph-search or \revadd{sampling-based} algorithms, such as the asymptotically-optimal RRT (RRT*) algorithm \cite{karaman2011sampling}. In the stochastic case \cite{ahmadi_risk-averse_2021,bertsekas2013stochastic}, however, the (now) SSP problem becomes possibly intractable, so that probabilistic completeness is inevitably lost, motivating \textit{probabilistically-constrained} sampling algorithms. One such class of probabilistic methods proceeds by formulating the problem as a mathematical program subject to chance constraints \cite{blackmore_probabilistic_2006, luders_chance_2010}, where the \revadd{obstacles, the robot's states, or both} are represented by variables with noise-perturbed parameters. Convex relaxations of the chance constraints generally follow such formulations to make the problem tractable. Unfortunately, such relaxations have been reported to be conservative \cite{liu_incremental_2014,ahmadi_risk-averse_2021}. To avoid such conservatism, some studies adopt risk-sensitive methods \cite{hakobyan_risk-aware_2019,nishimura_rap_2023} that introduce risk awareness by computing paths via \textit{risk-constrained} optimization, where risk is quantified by some measure, e.g., the Conditional Value-at-Risk (CVaR). 

Risk-aware techniques take into account highly-unlikely yet probable worst-case events, and, as a consequence, hedge against these tail outcomes to an extent proportional to their probabilities. Robustness to stochasticity naturally emerges from such methods, as demonstrated in the risk-sensitive optimization literature \cite{rockafellar_optimization_2000,noorani_embracing_2022}. In more related work \cite{ahmadi_risk-averse_2021}, the authors reformulate the SSP problem as Difference Convex Programs (DCPs) with the decision variable set as a dynamic risk functional and for Markov state transitions. Similar to our work, the authors adopt the CVaR (along with other risk measures); their method, however, relies on a convex reformulation of the SSP, a sound but needless step in our approach. Moreover, from a practical standpoint, DCPs may be challenging to implement owing to the many intricate details involved in their construction and their reliance on sophisticated optimization tools. In contrast, the modular nature of our Risk-Aware RRT* (RA-RRT*) algorithm makes it amenable to applications. Finally, unique to our work is also a discussion on the computational implications of considering risk in SSP planning as quantified by algorithmic complexity in processing time, query time, and memory space, a fundamental but mostly neglected consideration in studies focused on risk-aware SSP planning \cite{gavriel2012risk,ketkov2023multistage}. 

\subsection{Contributions \& Organization}
\begin{enumerate}[leftmargin=*,label=\roman*.]
    \setlength{\itemsep}{-4pt}
    \item \textit{Robust SSP Planning via Risk-Sensitive RRT*} (\cref{sec:ra-rrt-star}): We introduce the RA-RRT* algorithm for solving the SSP problem set forth in \cref{sec:prelim}. Our algorithm is a risk-sensitive adaptation of the RRT* algorithm to the stochastic setting that expands the search tree by selecting nodes at each sampling iteration to minimize an empirically-computed CVaR\footnote{This stage-wise minimization is key to the asymptotic optimality of the RRT* algorithm \cite{karaman2011sampling}, and, by extension, our adaptation.}. In addition, we characterize the computational efficiency of RA-RRT* (in \cref{ssec:algcmplx}) through comprehensive algorithmic analyses that compare the processing, query, and space complexity with the baseline RRT* algorithm. 
    \item \textit{Probabilistic Guarantee on the Optimal Worst-Case Path Length} (\cref{ssec:probub}): We derive an upper bound on the probability of the optimal worst-case path length exceeding a prescribed threshold for a given confidence level, which provides a formal confirmation of the probabilistic robustness of our risk-sensitive approach. 
    \item \textit{Simulation Results} (\cref{sec:resdis}): Finally, we present results from comprehensive simulation studies (\cref{sec:exp}) that demonstrate the utility of our risk-sensitive approach to incremental sampling-based path planning. These results form the basis of attendant elaborate discussions that further validate the theoretical (\cref{ssec:probub}) and computational complexity (\cref{ssec:algcmplx}) results.
\end{enumerate}

{\textbf{Notation}}: The symbol $[x]_i$ denotes the $\ith$ element of the vector, $x \in \mathbb{R}^d$, with transpose, $x^{\top}$, while $[A]_{ij}$ denotes the element in the $\ith$ row and $\jth$ column of the $m \times n$ matrix, $A \in \mathbb{R}^{m\times n}$, with $1 \le i \le m$, $1 \le j \le n$, and where $m$ and $n$ are both integers greater than 1. The symbol $\norm{ \ }$ will denote the Euclidean norm, and we will denote the expectation and variance of a random variable, $Y$, with realization $y$, by $\expt[Y]$ and $\variance[Y]$, respectively. Lastly, $\setr{a}{b}$ and $\setz{a}{b}$ respectively denote the set of real numbers and the integers on $[a,b]$, $\mathcal{A}[i]$ is the $\ith$ element of the set $\mathcal{A}$, \revadd{with cardinality, $|\mathcal{A}|$, and $P(E)$ denotes the probability of event $E$}.

\section{Background \& Problem Formulation}\label{sec:prelim}
\subsection{Finding Shortest Paths in Planar Regions with Obstacles}
We begin by stating the problem of finding shortest paths in $\mathbb{R}^2$ with convex obstacles. Useful background material follows that characterizes the configuration space and defines relevant path-related objects.
\begin{definition}[\revadd{Free Space Representation} \& SSP Problem Particulars]\label{def:xfreerep}\noindent Let $\mc{X} \subset \mathbb{R}^2$ denote the non-empty compact set representing the configuration space (C-space) of the robot and denote its obstacle-free subset as $\xfree := \{x \in \mathcal{X} \ | \ x \notin \Xobs \}$, where $\Xobs$ is the set of all (convex) C-space obstacles. We will take $\xfree$ to be represented by a tree, $\mc{T}$, given by the tuple $(\V_{\mc{T}}, \E_{\mc{T}})$, with $\V_\mc{T}$ and $\E_\mc{T} := \{(\cdot, \star) \ | \ \cdot, \star \in \V_\mc{T}\}$ respectively denoting the tree's node and edge sets. 
\end{definition}
\begin{definition}[Path]\label{def:path}\noindent Let $\xinit \in \V_{\mc{T}}$ and $\xgoal \in \V_{\mc{T}}$ respectively denote the prescribed start and target (goal) nodes. We define a path, $p \in \mathcal{P}$, as a finite sequence of $N$ edges in $\E_{\mc{T}}$ (hereafter, path segment) connecting $\xinit$ and $\xgoal$, i.e., $p$ is the set $ \{(\xinit, x_1), (x_1, x_2), \ldots, (x_{N-1}, \xgoal) \mid p[\revadd{k}] \in \E_{\mc{T}}, \revadd{k \in \setz{0}{N-1}}\})$, with each path segment (\revadd{i.e., $p[k]$}) lying entirely in $\xfree$. The symbol $\mathcal{P}$ denotes the set of all such paths. To \revadd{$p[k]$}, we associate a non-negative real number, $\Lagr_k := \revadd{\len{p[k]}}$, which quantifies \revadd{its length as determined by the operator $\len{}:\E_{\mc{T}} \rightarrow \setr{0}{\infty}$ (see \cref{ssec:uncrsk} for the definition of $\len{}$).}
\end{definition}
The following assumption on finite-time goal reachability is required to ensure that the SSP problem is well-posed and also serves as a terminating condition for our algorithm.
\begin{assumption}[Finite-Time Goal Reachability]\label{asm:finreach}
     There exists a finite integer, $\nmax$, such that the goal can be reached from any initial node in $\xfree$ in at most $\nmax$ time steps. 
\end{assumption}
To impose \cref{asm:finreach}, we require that the iteration step, $N$, at which the goal is reached be bounded above by $\nmax$. We can now state the SSP problem formally.
\begin{problem}[SSP Problem]\label{prb:ssp}\noindent Assume the setting in \cref{def:path,def:xfreerep} and \cref{asm:finreach}. Suppose also that \revadd{$\Lagr_k$} is unknown but follows a known probability distribution denoted by $\mathbb{P}^{\theta}_k$, with parameter, $\theta$. Find a path, $p$, of minimum expected length such that $x_0=\xinit$ and $x_N=\xgoal$\footnote{For simplicity, we have elected to formulate the SSP problem using a single goal point in $\xfree$ rather than a goal region, $\Xgoal$. In the latter case, our algorithm carries over as long as one sets $N = \inf \left\{k \in \mathbb{Z}_{0, \infty} \mid x_k \in \mathcal{X}_{\text {goal}}\right\}$.}\textsuperscript{,}\footnote{We note here that there might not exist a solution to \cref{prb:ssp}, since the path segment lengths can become arbitrarily large for unbounded $\theta$, necessitating a compact set assumption on the parameter space.}.
\end{problem}
Formally, we can express \cref{prb:ssp} as the following mathematical program%
\begin{subequations}\label{eq:optprob}
\begin{align}\label{seq:optprobobj}
&\hspace{0.08\columnwidth} \min_{p\in\mathcal{P}} \expt\left[\Lagr\right] \\
\text {s.t.: }&x_0 = \xinit,  x_{N} = \xgoal, N \le \nmax,\\
& x_k \in \xfree, \revadd{\Lagr_k = \len{p[k]}}, \ \forall \ k \in \setz{0}{\revadd{\nmax-1}},
\end{align}
\end{subequations}
with $\Lagr:=\sum_{k=0}^{N-1} \Lagr_k$ denoting the total path length (i.e., the \textit{cost}) and where the expectation is computed with respect to the \revadd{underlying} uncertainty distribution.

\subsection{Uncertainty Quantification \& Risk Notion}\label{ssec:uncrsk}
\textbf{Uncertainty Quantification}: As remarked in \cref{sec:intro}, the inherent stochasticity in the SSP problem stems from varied sources that typically influence the uncertainty model. For instance, localization inaccuracy may be captured via belief-based probability density functions \cite{thrun_probabilistic_2005} and dynamic obstacles via ambiguity sets \cite{hakobyan_distributionally_2023}. A canonical direction is typically to assume partial or noise-corrupted environmental information \cite{luders_chance_2010,blackmore_probabilistic_2006,liu_incremental_2014,li2023model}. Such characterizations of uncertainty besides enabling mathematical precision are also of immediate practical relevance due to their ease of implementation. These reasons motivate our adoption of this particular uncertainty class. Specifically, we take the length of the $\kth$ path segment to be given by the following expression 
\begin{equation}\label{eq:costevol}
    \Lagr_k = c + C_k, \quad C_k \sim \gauss{0}{\sigma^2_{C_k}},
\end{equation}
where $c := \normdiff{x_{k+1}}{x_k}$ is the Euclidean norm of the difference between configurations in $\xfree$ corresponding to consecutive nodes \revadd{in $\V_{\mc{T}}$}, and $C_k$ follows a \revadd{standard} normal distribution with variance $\sigma^2_{C_k}$. \revadd{Thus, we can define $\len{p[k]}$ as $\normdiff{\big[p[k]\big]_2}{\big[p[k]\big]_1} + C_k$, with $p[k] \equiv (x_{k}, x_{k+1})$.}

It is straightforward to verify that the distribution of $\Lagr_k$ is Gaussian, with mean, $c$, and variance, $\sigma^2_{C_k}$. We can thus write $\Lagr_k$'s probability density function and $\mathbb{P}^{\theta}_k$ respectively as
\begin{subequations}
\begin{align}\label{eq:pdfa}
f_{\Lagr_k}(\Lagr_k)&=\frac{1}{\sqrt{2 \pi \sigma^2_{C_k}}} \exp \left(-\frac{(\Lagr_k - c)^2}{2 \sigma^2_{C_k}}\right) \ \text{\revadd{and}}\\ \label{eq:pdfb}
\mathbb{P}^{\theta}_k &= \gauss{c}{\theta}, \quad \text{ with } \theta = \sigma^2_{C_k}.
\end{align}
\end{subequations}
\textbf{Risk Notion}: To assess the robot's risk at each stage of \cref{prb:ssp}, we adopt the $\alpha$-level CVaR ($\cvar$) of $\Lagr_k$, which quantifies the expected worst-case realization of $\Lagr_k$\revadd{, with $0 < \alpha \le 1$}. The $\cvar$ risk assessment tool possesses certain useful mathematical properties (see \cite{krokhmal_higher_2007} and attendant references), including convexity and monotonicity, which make it amenable to dynamic programming problems (such as the SSP problem) with optimal substructure. Intuitively, and in the path planning context, $\cvar(\Lagr_k)$ is the average of all realizations of $\Lagr_k$ that exceed the $\alpha$ quantile, i.e., the $\alpha$-level Value-at-Risk ($\var$), \revadd{which characterizes some threshold on $\Lagr_k$}. Furthermore, the CVaR provides probabilistic guarantees on the best (optimal) path length that can be attained under the prescribed maximum noise variance. Moreover, as set forth in the risk optimization literature \cite{james_risk-sensitive_1994} and, more recently, in results from risk-sensitive reinforcement learning \cite{noorani_risk-sensitive_2023}, by optimizing with respect to a risk-constrained objective function, one can establish a probabilistic guarantee on the upper bound of the (cumulative) cost tail probability (for bounded stage costs). These results motivate and lend credence to our risk-sensitive SSP planning approach.

\subsection[Computing the CVaR of the Segment-Wise Path Length]{Computing $\cvar(\Lagr_k)$}\label{ssec:compcvar}
We will now segue to deriving an expression for $\cvar(\Lagr_k)$ given the foregoing background. Formally, the $\cvar$ of a random variable $Y$ represents the expected value of the distribution of $Y$ conditioned on the event that $Y$ is above the $\var$ threshold. Mathematically, it is defined as
\begin{subequations}\label{eq:cvarform}  
    \begin{align}
    \label{eq:cvarforma} 
       \cvar(Y) &= \expt[Y \mid Y \geq \var(Y)]\\
                &=\inf_{z}\left\{z+\frac{1}{1-\alpha} E[Y-z]^{+}\right\}\\
                &=\frac{1}{1-\alpha} \int_\alpha^1 z dy,
    \end{align}
\end{subequations}
where $[\star]^+ := \max(\star, 0)$, and $z=\var(Y)$ \revadd{is the $\alpha$-quantile of $Y$'s distribution}, given by 
\begin{equation}\label{eq:varformalt}
    \var(Y) = \inf_{z} \{z \in \mathbb{R}\mid P(Y \leq z)\ge\alpha\}.
\end{equation}
In our current setting where the random variable of interest ($\Lagr_k$) is normally distributed, the $\var$ is given as
\begin{equation}\label{eq:varform}
    \var(\Lagr_k)=F_{\Lagr_k}^{-1}(\alpha).
\end{equation}
In \cref{eq:varform}, $F_{\Lagr_k}^{-1}(\alpha)$ is the inverse cumulative distribution function (CDF) of $\Lagr_k$ evaluated at $\alpha$, which we can write (from \cref{eq:pdfa}) as
\begin{equation}\label{eq:cdfform}
                 F_{\Lagr_k}^{-1}(\alpha)=c+\varsigma(\alpha) \sigma_{C_k},
\end{equation}
with $\varsigma(\alpha)=\sqrt{2}\operatorname{erf}^{-1}(2\alpha-1)$, and where $\operatorname{erf}(z)=(2 / \sqrt{\pi}) \int_0^z e^{-\tau^2} d\tau$ is the standard error function.
We can thus write $\cvar(\Lagr_k)$ from \cref{eq:cvarform,eq:cdfform} as:
\begin{equation}\label{eq:cvarformck} 
        \cvar(\Lagr_k)
        =\frac{1}{1-\alpha} \int_\alpha^1 \left(c+\varsigma(\alpha) \sigma_{C_k}\right) d\Lagr_k.
\end{equation}

\section{Robust SSP Planning via Risk-Sensitive RRT{*}}\label{sec:ra-rrt-star}
\subsection{SSP Planning with CVaR \texorpdfstring{\revadd{Criteria}}{Criteria}}\label{ssec:sppcvar}
Having discussed the calculation of the CVaR, we will now state a modified version of \cref{prb:ssp}, with the objective defined in terms of the CVaR. We first define the worst-case path length, which we will invoke in the modified problem.
\begin{definition}[Worst-Case Path Length]\label{def:wcpathlen}
We define the worst-case path length (denoted as $\lworst$) as the expected sum of the lengths of all path segments in $p$ under the maximum noise variance. Formally, we can define $\lworst$ as $\sup_{\sigma_{C_k}}\expt\left[\Lagr \right]$.
\end{definition}
To re-express $\lworst$ in terms of the CVaR, we \revadd{apply} the following dual representation of the CVaR (see \cite{ang2018dual})
\begin{equation}\label{eq:eqrepcvar}
\cvar(Y)=\sup_{\mathbb{P}_Y\in\Pi_Y}\expt[Y],
\end{equation}
where the supremum is taken over the set, $\Pi_Y$, of all admissible probability distributions, $\mathbb{P}_Y$ of $Y$.
With \cref{eq:eqrepcvar}, we can thus rewrite $\lworst$ as
\begin{equation}\label{eq:cvarlworsteqv}%
\lworst = \sup_{{\mathbb{P}_{\Lagr}}\in\Pi_\Lagr} \expt[\Lagr] = \cvar(\Lagr),
\end{equation}
where ${\mathbb{P}_{\Lagr}}$ is the probability distribution of the resulting path \revadd{length}, and $\Pi_\Lagr$ denotes the family of all such distributions\revadd{, i.e., each member of $\Pi_\Lagr$ characterizes an instance of ${\mathbb{P}_{\Lagr}}$ defined by a unique noise variance}. Thus, from the linearity of (conditional) expectation, under the assumption that the segment lengths are independent, and by the one-step coherence of the CVaR \cite{ahmadi_risk-averse_2021}, we can write
\begin{equation}\label{eq:cvarlworst}%
\lworst =\cvar(\Lagr)=\sum_{k=0}^{N-1} \cvar\left(\Lagr_k\right).%
\end{equation}
We can now state the modified SSP problem.
\begin{problem}[Modified SSP Problem with CVaR Criterion]\label{prb:sspcvar}\noindent Assume the setting in \cref{def:path,def:xfreerep,def:wcpathlen}. Find a path, $p$, that minimizes $\lworst$ such that $x_0=\xinit$ and $x_N=\xgoal$.
\end{problem}
Accordingly, the solution to \cref{prb:sspcvar} is the path $p$ that solves the following mathematical program%
\begin{subequations}\label{eq:optprobcvarobj}
\begin{align}\label{seq:optprobcvarobj}
&\hspace{0\columnwidth} \min_{p\in\mathcal{P}}\Jcvar(\xinit, \alpha) := \sum_{k=0}^{N-1} \cvar\left(\Lagr_k \right) \\
\text {s.t.: }&x_0 = \xinit,  x_{N} = \xgoal, N \le \nmax,\\
& x_k \in \xfree, \revadd{\Lagr_k = \len{p[k]}}, \ \forall \ k \in \setz{0}{\revadd{\nmax-1}}.
\end{align}
\end{subequations}

\subsection[Probabilistic Guarantee on the Optimal Worst-Case Path Length]{Probabilistic Guarantee on the Optimal \revadd{$\lworst$ Value}}\label{ssec:probub}
In \cref{prop:probub}, we derive a probabilistic bound on the optimal worst-case path length, motivated by the proof of Theorem 2 in \cite{noorani_risk-sensitive_2023}. We begin by deriving an expression for the Kullback–Leibler (KL) divergence in our present context via the following lemma. \revadd{As we will show, the KL divergence allows us to reason about the expected optimal worst-case path length through appropriate concentration inequalities\footnote{\revadd{As an alternative to the KL divergence, one can apply tools such as the total variation distance and the Wasserstein metric, with appropriate changes to the proofs.}}.}
\begin{lemma}\label{lem:kldiv}
 Let $\lstarworst$ denote the optimal \revadd{value of $\lworst$}, i.e., the \revadd{length of the path, $\operatorname{argmin}_{p\in\mathcal{P}} \Jcvar(\xinit, \alpha))$, returned by the mathematical program in \cref{prb:sspcvar}}, and suppose that ${\mathbb{P}_{\lstarworst}}$ is its associated \revadd{probability} distribution, \revadd{with mean, ${\muellstar}$, and variance, ${\sigmaellstarsq}$}. \revadd{Let $\lmax$ denote some designer-prescribed path length threshold}, with associated distribution, \revadd{${\mathbb{P}_{\lmax}} := \gauss{\lmax}{\sigma_\delta}$, where $\delta > 0$ is a tolerance on $\lmax$}\footnote{\revadd{For a $99\%$ confidence interval, one can define $\sigma_\delta \approx \nicefrac{\delta}{2.58}$, so that $P(\lmax-\delta \leq \Lagr \leq \lmax+\delta) \approx 0.99.$}}.%
  ~Denote the KL divergence of ${\mathbb{P}_{\lstarworst}}$ from ${\mathbb{P}_{\lmax}}$ as $\revadd{\mathbb{D}_{\mathrm{KL}}\left({\mathbb{P}_{\lstarworst}} \ \| \ {\mathbb{P}_{\lmax}}\right)}$,\revadd{ and suppose there exists a sufficiently small $\epsilon > 0$} such that $\revadd{\mathbb{D}_{\mathrm{KL}}\left({\mathbb{P}_{\lstarworst}} \ \| \ {\mathbb{P}_{\lmax}}\right)} \le \varepsilon$. Then the following inequality holds:
\begin{equation}
    \log \frac{\sigmaellstar}{\sigma_{\revadd{\delta}}\sqrt{N}}+\frac{N\sigma^2_{\revadd{\delta}}+\left(\sum_{k=0}^{N-1}c-\muellstar\right)^2}{2 \sigmaellstarsq}-\frac{1}{2} \le \varepsilon.
\end{equation}
\end{lemma}%
\begin{proof}
By the central limit theorem for sums, we know that in the limit as $N$ increases, \revadd{$\mathbb{P}_{\lstarworst}$} \revadd{will converge} in probability to a normal distribution given by
${\mathbb{P}_{\Lagr}} = \gauss{\muell}{\sigmaellsq}$, where $\muell := \sum_{k=0}^{N-1}c$ and $\sigmaell := \sigma_{\revadd{\delta}}\sqrt{N}$. \revadd{By} the definition of the KL divergence between two univariate Gaussians, we can write $\revadd{\mathbb{D}_{\mathrm{KL}}\left({\mathbb{P}_{\lstarworst}} \ \| \ {\mathbb{P}_{\lmax}}\right)}$ as\begin{align}&\log \frac{\sigmaellstar}{\sigmaell}+\frac{\sigmaellsq+\left(\muell-\muellstar\right)^2}{2 \sigmaellstarsq}-\frac{1}{2},\\
&\Rightarrow \log \frac{\sigmaellstar}{\sigma_{\revadd{\delta}}\sqrt{N}}+\frac{N\sigma^2_{\revadd{\delta}}+\left(\sum_{k=0}^{N-1}c-\muellstar\right)^2}{2 \sigmaellstarsq}-\frac{1}{2}\le \varepsilon.\nonumber\end{align}\end{proof}%
\begin{proposition}[Probabilistic Guarantee on \revadd{$\lstarworst$}]\label{prop:probub}Suppose the premises of \cref{lem:kldiv} hold. Then the \revadd{probability of $\lstarworst$ exceeding $\lmax$ is bounded above by the quantity specified in the} following inequality:
\begin{subequations}
\begin{align}
&P_{\revadd{\lstarworst} \sim {\revadd{\mathbb{P}_{\lstarworst}}}}[\lstarworst \geq \lmax] \nonumber\\
&\leq  \frac{1}{\lmax}\revadd{\displaystyle\sum_{k=0}^{|p|-1}{\len{\Jcvar^\star[k]}}}+\frac{1}{\alpha }\frac{\varepsilon}{\lmax}, \quad 0 <\alpha\le1 \nonumber,
\end{align}
\end{subequations}
where \revadd{$\Jcvar^\star:=\arg \min_{p\in\mathcal{P}}\Jcvar(\xinit, \alpha)$}.
\end{proposition}
\begin{proof}
Invoking Markov's inequality, we can write
\begin{align}\label{eq:bound1}
&P_{\revadd{\lstarworst \sim {\mathbb{P}_{\lstarworst}}}}[\lstarworst \geq \lmax] \leq \frac{\mathbb{E}[\lstarworst]}{\lmax}\\
& \leq \frac{1}{\lmax}\left(\revadd{\displaystyle\sum_{k=0}^{|p|-1}{\len{\Jcvar^\star[k]}}}+\revadd{\frac{1}{\alpha} \mathbb{D}_{\mathrm{KL}}\left({\mathbb{P}_{\lstarworst}} \ \| \ {\mathbb{P}_{\lmax}}\right)}\right)\nonumber.
\end{align}
From \cref{eq:bound1} and by \cref{lem:kldiv}, we can thus write:
\begin{align}\label{eq:bound2}
&P_{\revadd{\lstarworst \sim {\mathbb{P}_{\lstarworst}}}}[\lstarworst \geq \lmax] \nonumber\\
&\quad \leq \frac{1}{\lmax}{\revadd{\displaystyle\sum_{k=0}^{|p|-1}{\len{\Jcvar^\star[k]}}}}+ \frac{1}{\lmax}\frac{1}{\alpha}\varepsilon, \text{ \revadd{so that} }
\end{align}%
\revadd{%
\begin{equation}\label{eq:ubprobbound}\mathbb{E}[\revadd{\lstarworst}] \leq \revadd{\displaystyle\sum_{k=0}^{|p|-1}{\len{\Jcvar^\star[k]}}}+\frac{1}{\alpha}{\varepsilon}.\end{equation}
}\end{proof}
\subsection{The RA-RRT* Planning Algorithm}\label{ssec:rarrtstar}
In this section, we will discuss the components of our RA-RRT* algorithm for solving \cref{prb:sspcvar}. \cref{fig:grid-world-ex} highlights relevant running components that will be referenced in this section. In what follows, we will assume the tree-based representation of $\xfree$ set forth in \cref{def:xfreerep}. 
\renewcommand*\footnoterule{}
\begin{savenotes}
\begin{algorithm}[ht]
    \caption{RA-RRT*: Tree Growth}\label{alg:rarrt}
    \begin{algorithmic}[1]
        \State \Input Empty tree: $\mathcal{T}$, with start position, $\xinit$ as root; Goal position: $\xgoal$.
        \State \Parameters Maximum number of iterations: $\maxiter$, Neighborhood radius: $\rneighbors$; Maximum rewiring radius: $\maxrewirerad$; Robot radius: $\rrad$; Variance schedule of path-segment-wise additive noise: $\sigma^2_{C_k}, \ \forall \ k \in \setz{0}{\maxiter}$; Number of samples for $\kth$ segment random cost distribution: $\numcostdistsamp$; Confidence level: $\alpha$; Number of neighbors: $k_n$.
        \State \Output Path, $p$.
        \State $x_0 \gets \xinit$, $p \gets [\xinit]$.
        \While{$k \le \maxiter-1$}
            \State $\xrand \gets$ {\scshape{Sample}}($\xfree$).
            \State $\xnearest \gets$ \scs{Nearest}($\mc{T},\xrand$).
            \State $\xnew \gets$ {\scshape{Steer}}($\xnearest, \xrand, \mc{T}$).
            \LComment{If candidate segment in $p$ is obstacle-free:}
            \If{\scs{\revadd{Line}}($\xnew$, $\xnearest$) $\notin \Xobs$}
                \State $N(\xnew) \gets ${\scshape{GetNeighbors}}$(\xnew, \rneighbors, k_n)$.\label{algline:getneig}
                \State $V_{\mc{T}} \gets \{\xnew\}$
                \ForAll{$\xnear \in N(\xnew)$}
                    \State Compute $c(\xnearest, \xnear)$ and construct distribution \big($\gauss{c}{\sigma^2_{C_k}}$\big) of $\numcostdistsamp \ \Lagr_k$ samples.%
                    \State Compute $\var(\Lagr_k)$ \Comment{\small Equation \cref{eq:varform}}.\label{algline:var}
                    \State Compute $\cvar(\Lagr_k)$ \Comment{\small Equation \cref{eq:cvarformck}}.\label{algline:cvar}
                \EndFor
                \State Set $\xmincvar := \arg\inf_{\xnear}\cvar(\Lagr_k)$.
                \LComment{See \cref{alg:checkcollision} for *\tee{\small args}.}
                \If{{\scshape{\revadd{CollisionFree}}}($\xmincvar, \xnew, *\tee{\small args}$)}
                    \State $p \gets \{p, \revadd{(\xmincvar, \xnew)}\}$.
                    \State $\mc{T} \gets$ {\scshape{Rewire}}($\xmincvar, \maxrewirerad, \mc{T}$).
                \Else 
                   \State $\Xobs \gets \Xobs \bigcup \revadd{\{\xmincvar\}}$.
                   \State {Return to} Line 5.
                \EndIf
                \State $\xmincvar \leftarrow x_k$.
                \State $k\leftarrow k+1$.
                \If{$k = \maxiter- 1$ \textbf{and} $x_k \neq \xgoal$}
                    \State \Return {{\scshape{Failure}}}.
                \EndIf
            \EndIf
        \EndWhile
        \Return $p$.
    \end{algorithmic}
\end{algorithm}
\end{savenotes}
\begin{figure}[htb]
    \def\figW{0.32}
    \def\scaleW{.7}
    \centering
    \scalebox{\scaleW}{
    \begin{tikzpicture}[spy using outlines={circle,black,magnification=2.5,size=3.2cm, connect spies,every spy on node/.append style={ultra thick}}]
    \node {\includegraphics[interpolate,width=\columnwidth]{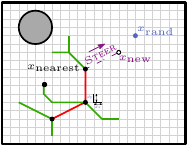}};
    \spy[every spy on node/.append style={ultra thick}] on (-0.15,-1.42) in node [label={[shift={(0.15,-5.3)}]{\Large\color{arxblue}$\bm{2.5\times}$}},right] at (4.5,0);
    \end{tikzpicture}
    }
    \caption{\textbf{Planning environment}: An illustration of the \revadd{grid-world} environment capturing important objects used throughout this article. Tree edges are depicted in green color and path segments in red color, while the gray-filled circle denotes an obstacle. A magnified inset highlights the grid's resolution along the abscissa and ordinate axes.}\label{fig:grid-world-ex}
\end{figure}

\cref{alg:rarrt} contains pseudocode for implementing our RA-RRT* algorithm that proceeds as follows. First, a random node (denoted as $\xrand$) is sampled uniformly from $\xfree$ via the \scs{Sample} primitive routine\footnote{In practice, $\xrand$ is sampled from a prescribed area bounding the most current node in $\V_{\mc{T}}$.}. Next, the nearest node ($\xnearest$) in $\mc{T}$ to $\xrand$ is computed via the \scs{Nearest} subroutine by finding the element in $\V_\mc{T}$ satisfying the relation $\inf_{x\in \V_{\mc{T}}}\normdiff{x}{\xrand}$. 
Having computed $\xnearest$, the {{\scshape{Steer}}} primitive routine returns a node ($\xnew$) in $\xfree$ that is closer to $\xrand$ than $\xnearest$, i.e., it attempts to drive $\xnearest$ towards $\xrand$ \cite{karaman2011sampling}. 
The {\scshape{GetNeighbors}} sub-procedure computes and returns a set of $k_n$ \textit{nearest} nodes \revadd{(to $\xnew$)} within the ball of radius, $\rneighbors$, i.e., the set, $N(\xnew):= \{x^{\prime} \in \xfree \ | \ \normdiff{\xnew}{x^{\prime}} \le \rneighbors\}$. This step is then followed by a computation of the $\var$ and $\cvar$ of the \revadd{length} associated with the path segment formed by $\xnear$ and $\xnew$ for each node $\xnear$ in $N(\xnew)$. The minimum-CVaR node, $\xmincvar$, over all nodes in $N(\xnew)$ is subsequently determined.
\begin{algorithm}[t]
    \caption{{\scshape{\revadd{CollisionFree}}}($\xmincvar, \xnew, \Xobs, \rrad, \mc{T}$)}\label{alg:checkcollision}
    \begin{algorithmic}[1]
        \State \Inputs $\xmincvar, \Xobs, \rrad, \mc{T}$.
        \State \Output Boolean: \tee{True}, \tee{False}.
        \For{$\mc{X}_{i,o}, i = 1, 2, \ldots, |\Xobs|$}
            \If{\scs{NonConvex}($\mc{X}_{i,o}$)} \Comment{Boolean-valued sub-routine that returns \tee{True} if $\mc{X}_{i,o}$ is non-convex.}
                \State $\Xobs \setminus \mc{X}_{i,o}$.
                \State $\Tilde{\mc{X}}_{i,o} \gets$ {\small\scs{PolyhedralApproximation}($\mc{X}_{i,o},\numobsappx$)}.\label{algline:pappx}
            \EndIf
            \State $\Xobs \gets \Xobs \bigcup_{\forall i}\Tilde{\mc{X}}_{i,o}$.
        \EndFor
        \If{\scs{\revadd{Line}}$(\xmincvar, \xnew)\notin\Xobs$}\Comment{If edge is collision-free}
           \State {\Return \tee{True}}.
        \EndIf
        \State {\Return \tee{False}}.
    \end{algorithmic}
\end{algorithm}
Upon discovery of $\xmincvar$, we invoke the {\scshape{\revadd{CollisionFree}}} primitive routine (\cref{alg:checkcollision}) to check if the path segment between $\xmincvar$ and $\xnearest$ lies entirely in $\xfree$. This subroutine assumes that the obstacles \revadd{($\mc{X}_{i,o}$) in $\mc{X}_{o}$} are (or can be made) convex (see line~\ref{algline:pappx}) by invoking the \scs{PolyhedralApproximation} procedure\footnote{The \scs{PolyhedralApproximation} sub-procedure --- adapted from \cite{marcucciShortestPathsGraphs2024a} --- computes and returns a polyhedral approximation \revadd{($\Tilde{\mc{X}}_{i,o}$)} of a non-convex obstacle \revadd{($\mc{X}_{i,o}$) given $\mc{X}_{i,o}$'s desired number of vertices, $\numobsappx$}.}, so that given the robot's radius, $\rrad$, the obstacles can be dilated using a dilation technique such as the Minkowski sum. 

\revadd{If no collision is detected}, the {\scshape{Rewire}} subroutine checks nodes within a ball of the rewiring radius, $\rewirerad$, to determine if a better (shorter) path to $\xinit$ can be traced from the newly-added node ($\xnew$) and updates the tree accordingly if a node is found. At runtime, $\rewirerad$ is determined as $\min \left\{\gamma(\nicefrac{\log {|\V_\mc{T}|}}{|\V_{\mc{T}}|})^{1 / 2}, \maxrewirerad\right\}$, where $\maxrewirerad$ is the maximum rewiring radius, and $\gamma$ is a scalar that is calculated using an expression provided in \cite{karaman2011sampling} (Section 3.3.1). We refer the reader to \cite{karaman2011sampling} for a thorough discussion on the foregoing procedures.

\subsection{Algorithmic Complexity of the RA-RRT* Algorithm}\label{ssec:algcmplx}
In this section, we present the computational complexity of our RA-RRT* algorithm. As established in the path planning literature (\cite{lavalle_planning_2006,karaman2011sampling}), it is well known that the RRT* algorithm is linear in query time ($\mathcal{O}(n)$) and log-linear in processing time ($\mathcal{O}(n\log{n})$)), while its (memory) space complexity is linear ($\mathcal{O}(n)$), where $n$ is the sample size. Our algorithm differs from the RRT* algorithm in two ways. First, in the generation of the set of nearby nodes (line~ \revadd{\ref{algline:getneig}}), connections to $k_n$ nearest neighbors are sought. The other difference is in the RA-RRT* algorithm's node selection process that introduces $\var$ and $\cvar$ computation steps (lines~\revadd{\ref{algline:var}}~and~\revadd{\ref{algline:cvar}} of \cref{alg:rarrt}) prior to the \scs{\revadd{CollisionFree}} sub-procedure. With this background and assuming the sorting algorithm for each sorting-dependent quantity (e.g., $\var, \cvar$, $\min$, etc.) is the \textit{merge sort} algorithm, we can affirm the following:
\begin{enumerate}[leftmargin=*,label=\roman*.]
    \setlength{\itemsep}{-5pt}
    \item The RA-RRT* algorithm is \textit{linear} in query time, since it is a single-query algorithm like the RRT*.  
    \item The RA-RRT* algorithm has a space complexity of $\mc{O}(n)$, i.e., it is \textit{linear-space}, since only the path is saved in memory for single-query access.
    \item From lines~\revadd{\ref{algline:var}}~and~\revadd{\ref{algline:cvar}} of \cref{alg:rarrt}, with $\numcostdistsamp$ corresponding to the number of samples of $C_k$ drawn at the $\kth$ sampling iteration and for $k_n$ neighbors ($\forall \ k$), the processing time complexity of the RA-RRT* algorithm can be calculated as $\mc{O}(n\log{(n\cdot\nrarrt)})$, where $\nrarrt$ is the number determined by the following expression
    \begin{equation}\label{eq:nrarrt}
        \max\left(\ncvark, \numcostdistsamp\right).
    \end{equation}
    Thus, the RA-RRT* algorithm is also \textit{log-linear} in processing time.
\end{enumerate}
    In \cref{eq:nrarrt}, $\ncvark$ is the cardinality of the following set 
    \begin{subequations}
    \begin{align*}
    \{\xnear \in N(\xnew) \ | \ \normdiff{\xnear}{\xnew} + C_k \ge \var(\Lagr_k)\}.
    \end{align*}
    \end{subequations}
    Thus, by comparing the processing time complexity for both algorithms, we expect an increase in the computation time (see \cref{tab:allresults}, Column 3), since for a positive base (base-2 in this case) and for $n$ and $\nrarrt$ both greater than 1,
    \begin{subequations}
    \begin{align}
           \mc{O}(n\log{(n\cdot\nrarrt)}) &= \mc{O}(n[\log {n} + \log{\nrarrt}]) > \mc{O}(n\log{n}).\nonumber
    \end{align}
    \end{subequations}
    We note here that $n$ in the foregoing algorithmic analysis corresponds to $\maxiter$ used in the main algorithm; we have simply favored the more concise notation here for brevity.

\section{Numerical Simulations}\label{sec:exp}
\subsection{Planning Environment}
We represent the planning environment by a (planar) grid of finite area and sufficient discretization, where each grid point corresponds to a unique \revadd{(robot)} configuration in $\mc{X}$. Configurations belonging to C-space obstacles within the grid are known and assumed to be dilated by the prescribed robot radius, $\rrad$, so that a point-robot assumption makes sense geometrically. To each new node in the tree, we associate a $k_n$-connected neighborhood, that is, from each grid point, the robot can attempt to steer towards the uniformly-sampled node from one of $k_n$ nearby configurations. Transitions in any given direction are permissible as long as the Euclidean distance to be traversed is under the prescribed threshold, $\rewirerad\sqrt(\dgridx^2 + \dgridy^2)$\revadd{, where $\rewirerad$ is the rewiring radius (see the discussion immediately preceding \cref{ssec:algcmplx})}. It is straightforward to verify that such motion constraints can be represented by the single-integrator model
\begin{equation}\label{eq:singintdk}
        x_{k+1} = x_k + u_k,
\end{equation}
where $u_k$ is in the set
\[\setlength\arraycolsep{2pt}
    \left\{
        \begin{bsmallmatrix}
            \vertmat{\pm\dgridx[\kappa]_1}{0}, &
            \vertmat{0}{\pm\dgridy[\kappa]_2}, &
            \vertmat{\pm\dgridx[\kappa]_1}{\dgridy[\kappa]_2}, &
            \vertmat{\pm\dgridx[\kappa]_1}{-\dgridy[\kappa]_2}
        \end{bsmallmatrix}
        \right\}
\]
and $\kappa = \op{abs}({x_{k+1}}-{x_k})$, with $\op{abs}$ denoting the (element-wise) absolute value. For each value of $\alpha \in \{0.1, 0.5, 0.9\}$, we conducted 50 independent runs of the RRT* and RA-RRT* algorithms with five (circular) convex obstacles and a sixth non-convex obstacle formed by joining three circles of different radii. In addition, we performed simulation runs for $\sigma_{C_k} \in \{0.01, 0.05, 0.1, 0.5\}$. We selected these noise variance values using heuristics, along with the SSP problem's particulars, to capture a sufficiently-varied \revadd{uncertainty spectrum}. We also found that, for $\sigma_{C_k} << 0.01$, the noise becomes too imperceptible to be relevant in our analysis, and, conversely, for high noise variance (i.e., for values of $\sigma_{C_k} >> 0.5$), the problem becomes increasingly intractable.

\section{Results \& Discussions}\label{sec:resdis}
In this section, we discuss results recorded from our numerical experiments under representative subheadings that come next. \revadd{In \cref{tab:allresults},} we provide a comprehensive (quantitative) summary of our results. There, the {\scs{Failure Rate}} column header records the percentage of planner failure measured by how many times (\revadd{in 50 runs}) the planner failed to return a path within the maximum number of sampling iterations.
\begingroup
\begin{table}[t]
\def\firstcolW{*}
\centering
\caption{\sc Comparing the Path Lengths and Failure Rates of the RRT* and RA-RRT* Algorithms with Increasing Noise and for Different Risk Confidence Levels}\label{tab:allresults}
\begin{adjustbox}{max width=\columnwidth}
\begin{tabularx}{\columnwidth}{C{.8cm}C{1.1cm}C{1.2cm}C{.4cm}C{1cm}C{1.2cm}}
\toprule {{Algorithm}} & $\Lagr^*$ $\unit[per-mode = symbol]{[\meter]}$ & Computation Time $[s]$ & $\sigma_{C_k}$ &{\scs{Failure Rate}} & $\Lagr^\star_{\text{worst}} \pm \sigmaellstarsq [m]$ \\
\toprule \toprule \multirow{4}{\firstcolW}{\scriptsize $\text{RRT*}$} & 5.29 & 0.0143 & 0.01 & $16 \%$ & \multirow{4}{*}{\small $5.42\pm0.126$} \\
 & 5.26 & 0.0124 & 0.05 & $10 \%$ & \\
 & 5.22 & 0.0136 & 0.1 & $ 12 \%$ & \\
 & 5.41 & 0.0180 & 0.5 & $ 14 \%$ &  \\
\midrule \multirow{4}{\firstcolW}{\scriptsize $\text{RA-RRT*}_{0.1}$} & 5.29 & 0.21 & 0.01 & $6 \%$ & \multirow{4}{*}{\small $5.29 \pm 0.088$}\\
 & 5.25 & 0.29 & 0.05 & $6 \%$ &  \\
 & 5.22 & 0.24 & 0.1 & $6 \%$ & \\
 & 5.25 & 0.27 & 0.5 & $6 \%$ & \\
 \midrule \multirow{4}{\firstcolW}{\scriptsize $\text{RA-RRT*}_{0.5}$} & 5.26 & 0.26 & 0.01 & $8 \%$ & \multirow{4}{*}{\small $5.34 \pm 0.091$}\\
 & 5.29 & 0.28 & 0.05 & $8 \%$  & \\
 & 5.20 & 0.27 & 0.1 & $2 \%$  & \\
 & 5.34 & 0.19 & 0.5 & $6 \%$  & \\
 \midrule \multirow{4}{\firstcolW}{\scriptsize $\text{RA-RRT*}_{0.9}$} & 5.27 & 0.22 & 0.01 & $8 \%$ & \multirow{4}{*}{\small $5.32 \pm 0.103$} \\
 & 5.25 & 0.27 & 0.05 & $8 \%$  & \\
 & 5.18 & 0.26 & 0.1 & $4 \%$  & \\
 & 5.32 & 0.27 & 0.5 & $6 \%$ & \\
\toprule
\end{tabularx}
\end{adjustbox}
\end{table}
\endgroup
\subsection{Worst-Case SSP Planning Performance}\label{ssec:wcsspperf}
On \cref{tab:allresults}, we compare the mean and variance of the worst-case \revadd{path} lengths (i.e., path lengths corresponding to $\sigma_{C_k} = 0.5$) over 50 runs. From here, for the RA-RRT* algorithm, we notice a smaller variability in the mean path length as well as reduced variance, as opposed to the RRT* baseline. In addition, by comparing the trees and paths generated by both algorithms (see \cref{fig:overall}) with increasing noise variance and for a fixed confidence level ($\alpha = 0.9$), we notice a decrease in the connectivity of the generated nodes that ultimately leads to prolonged processing times and longer paths for the baseline. In contrast, for the RA-RRT* algorithm, we observe longer computation times but (shorter) gracefully-degraded path lengths. Lastly, by examining the failure rate column, we can infer that the RA-RRT* fails significantly less, even under increasing noise, unlike the RRT*. Finally, by inspecting the final column of \cref{tab:allresults}, we also notice reduced worst-case path lengths for the RA-RRT* \revadd{algorithm} than the baseline. 
\begin{figure*}[htb]
        \centering 
        \def\figW{0.32}
        \def\scaleW{1.1}
        \begin{subfigure}[b]{\figW\textwidth}
            \scalebox{\scaleW}{
            \begin{tikzpicture}[spy using outlines={circle,black,magnification=4,size=2cm, connect spies,every spy on node/.append style={ultra thick}}]
            \node {\includegraphics[interpolate,width=\textwidth]{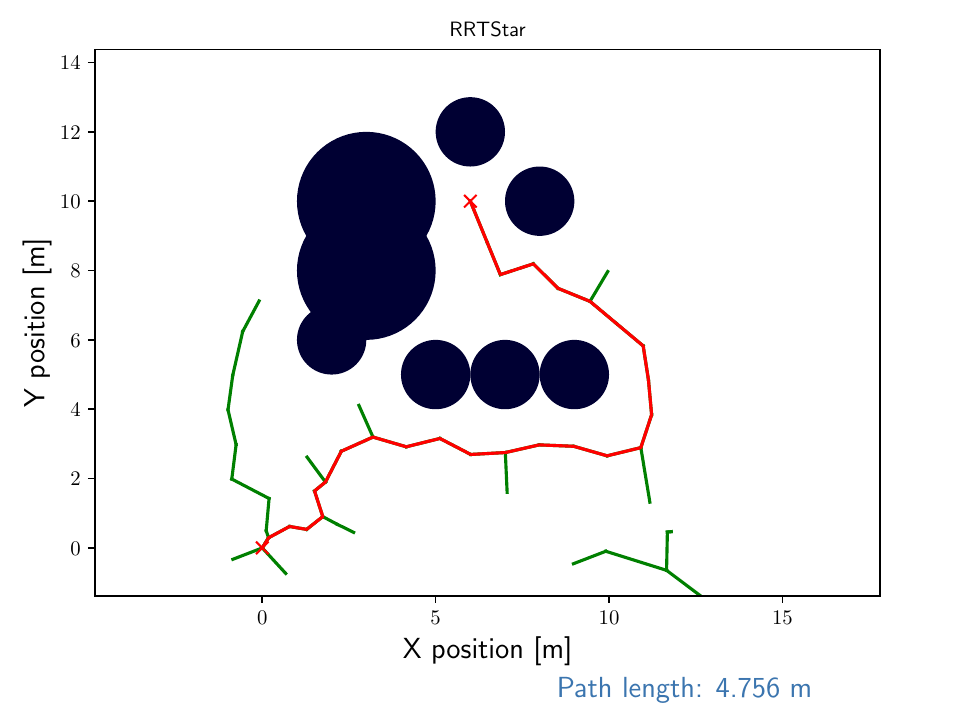}};
            \spy[every spy on node/.append style={ultra thick}] on (1.05,-0.9) in node [label={[shift={(0.19,-3.2)}]\small {\color{arxblue}$\bm{4\times}$}},left] at (2.4,1.8);
            \end{tikzpicture}}
            \caption{$\sigma_{C_k} = 0.05, \ \Lagr^\star = 4.756$ m.}\label{fig:rrtalpha09sigma05}
        \end{subfigure}
        \hfill
        \begin{subfigure}[b]{\figW\textwidth}
            \scalebox{\scaleW}{\includegraphics[interpolate,width=\textwidth]{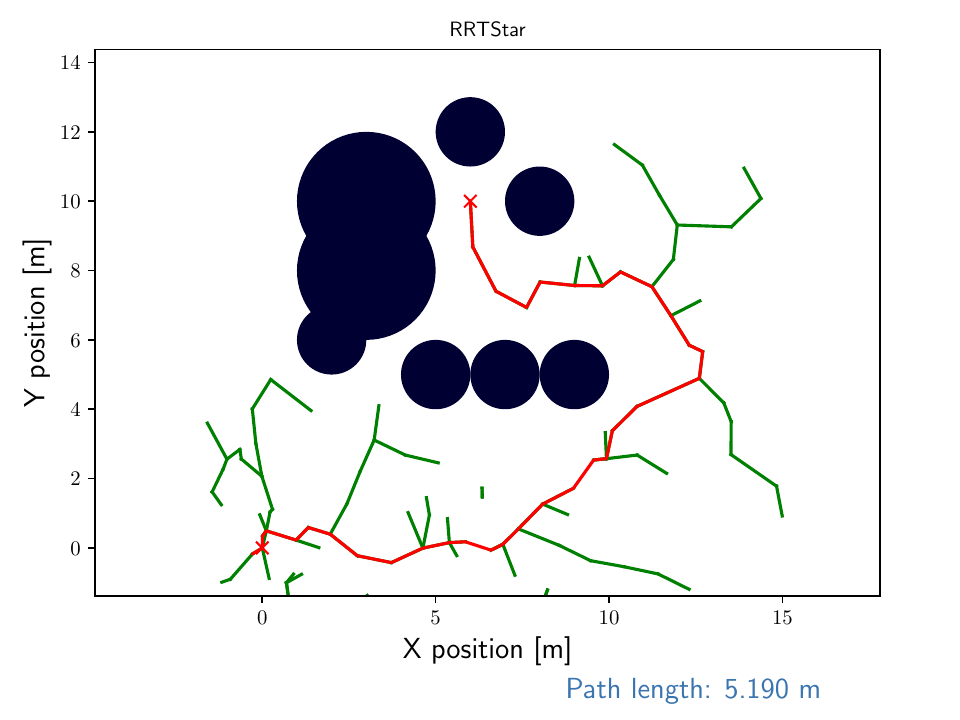}}
            \caption{$\sigma_{C_k} = 0.1, \ \Lagr^\star = 5.19$ m.}\label{fig:rrtalpha05sigma05}
        \end{subfigure}
        \hfill
        \begin{subfigure}[b]{\figW\textwidth}
            \scalebox{\scaleW}{
            \begin{tikzpicture}[spy using outlines={circle,black,magnification=5,size=1.8cm, connect spies,every spy on node/.append style={ultra thick}}]
            \node {\includegraphics[interpolate,width=\textwidth]{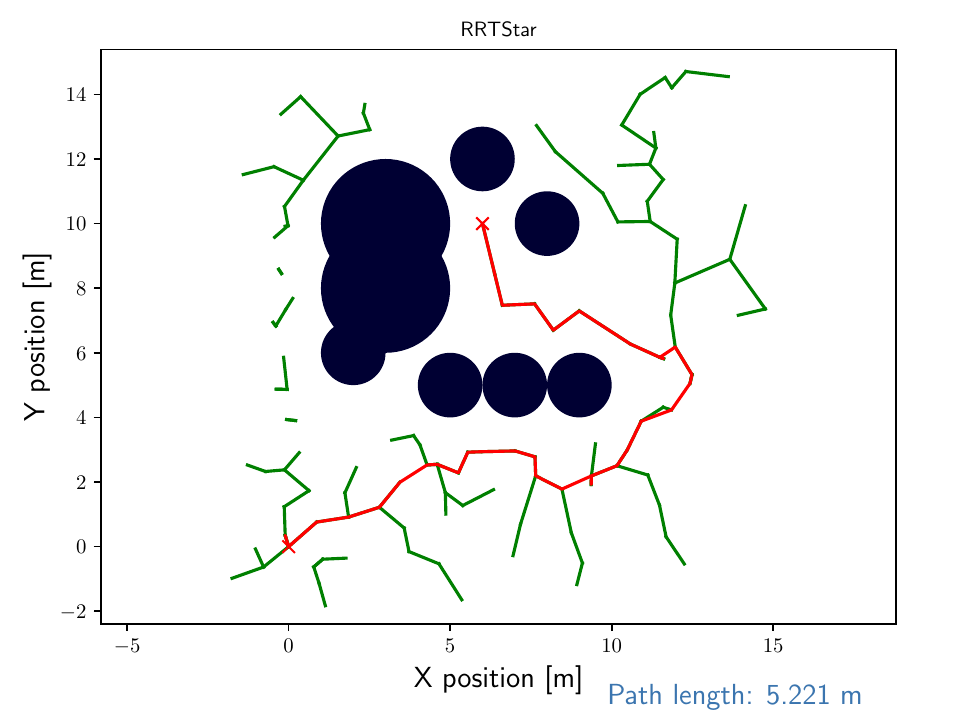}};
            \spy[every spy on node/.append style={ultra thick}] on (-1.2,0.42) in node [label={[shift={(0.15,-3)}]\small {\color{arxblue}$\bm{5\times}$}},left] at (2.8,1.5);
            \end{tikzpicture}
            }
            \caption{$\sigma_{C_k} = 0.5, \ \Lagr^\star = 5.221$ m.}\label{fig:planrrtalpha01sigma05}
        \end{subfigure}
        \\
        \begin{subfigure}[b]{\figW\textwidth}
            \scalebox{\scaleW}{\includegraphics[interpolate,width=\textwidth]{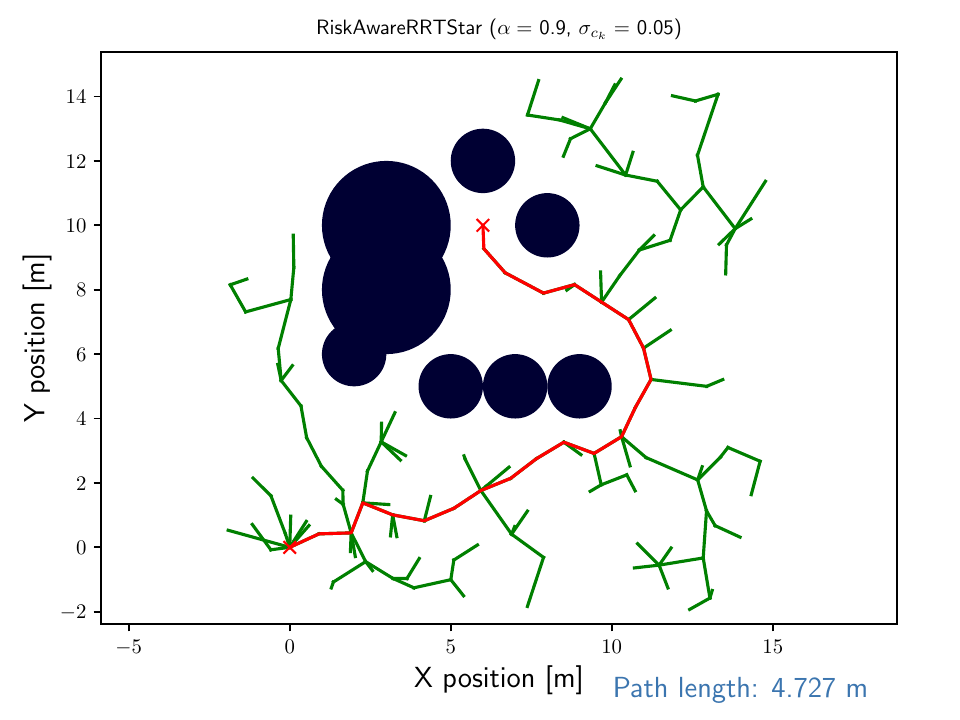}}
            \caption{$\sigma_{C_k} = 0.05, \ \Lagr^\star = 4.727$ m.}\label{fig:planrarrtalpha09sigma05}
        \end{subfigure}
        \hfill
        \begin{subfigure}[b]{\figW\textwidth}
            \scalebox{\scaleW}{\includegraphics[interpolate,width=\textwidth]{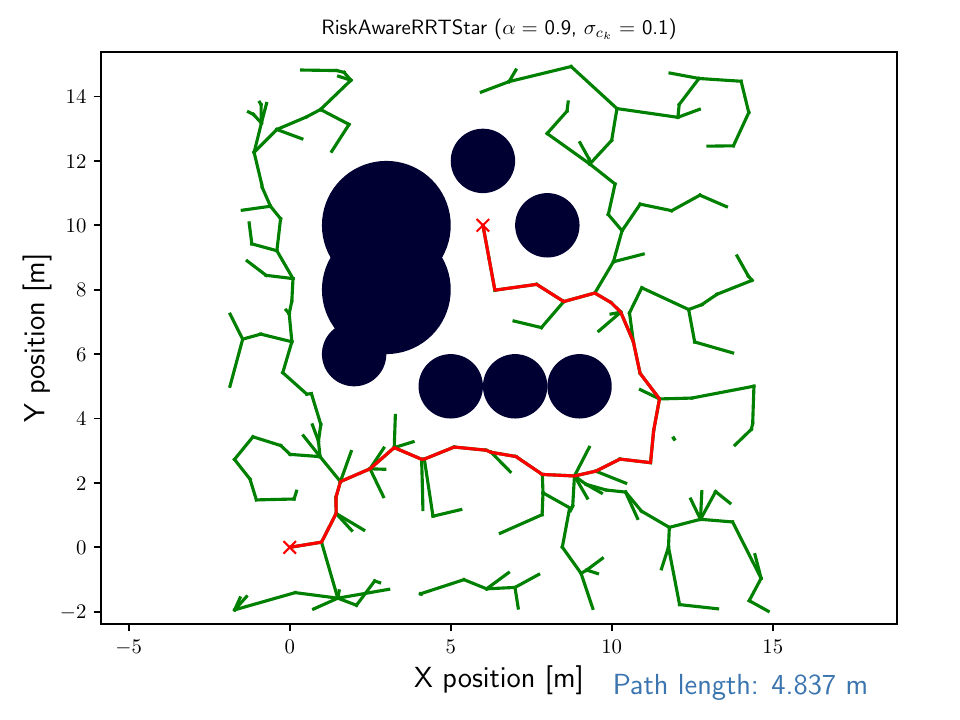}}
            \caption{$\sigma_{C_k} = 0.1, \ \Lagr^\star = 4.837$ m.}\label{fig:planrarrtalpha05sigma05}
        \end{subfigure}
            \hfill
        \begin{subfigure}[b]{\figW\textwidth}
            \scalebox{\scaleW}{\includegraphics[interpolate,width=\textwidth]{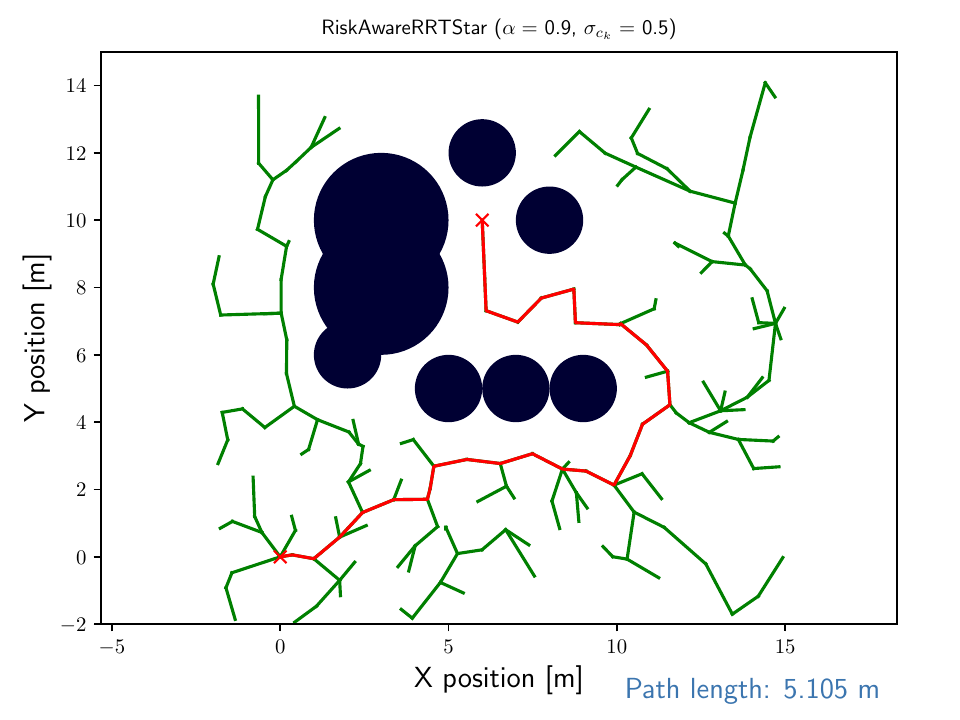}}
            \caption{$\sigma_{C_k} = 0.5, \ \Lagr^\star = 5.105$ m.}\label{fig:planrarrtalpha01sigma05}
        \end{subfigure}
\caption{\textbf{Visualizing the planned paths}: Example shortest paths returned by the RRT* (\textbf{top}) and RA-RRT* (\textbf{bottom}) algorithms for a fixed value of risk-sensitivity ($\alpha = 0.9$) parameter and increasing stochasticity ($\sigma_{C_k}$, \textbf{left-to-right}). Here, we see that, as the noise parameter is increased, the RA-RRT* planner's performance degrades gracefully with increasing uncertainty (evidenced by the slowly-increasing optimal path length, $\Lagr^\star$), while the RRT* planner returns infeasible paths (of greater lengths than the RA-RRT*) connecting configurations on a disconnected tree (see the magnified insets).}\label{fig:overall}
\end{figure*}

\subsection{Assessing the Cost of Risk-Sensitive SSP Planning}\label{ssec:costofrisk}Furthermore, from \cref{fig:overall} and \cref{tab:allresults}, we notice that, although the RA-RRT* algorithm takes slightly more computation time (in the order of $16.7, 16.8$, and $16.8$ for $\alpha= 0.1, 0.5$, and $0.9$, respectively), it is markedly less sensitive to variations in the noise parameter, evidenced by its shorter path lengths and reduced percentage of failure with increasing noise. We also observe from \cref{fig:mean_var}, that, under the RA-RRT* algorithm, the variance in the path length is equal or lower than that of the RRT*, for increasing stochasticity and for all three $\alpha$ values.
\begin{figure}[htb]
    \centering    \includegraphics[interpolate,width=.5\textwidth]{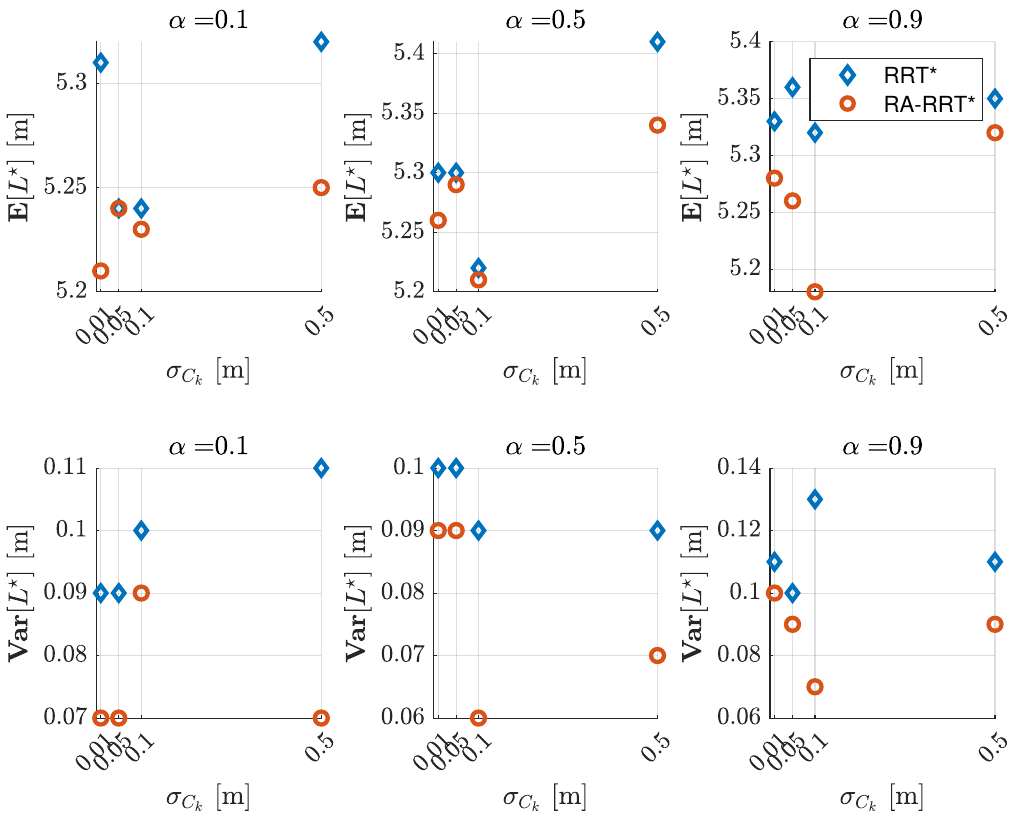}
    \caption{Mean and variance of the shortest path lengths obtained for different $\alpha$ and $\sigma_{C_k}$ values. The values of the mean path length and variance for the RA-RRT* algorithm ({\color{mplred} red $\circ$}) are either equal or better (lower) than that of the RRT* ({\color{mplblue} blue $\lozenge$}).}\label{fig:mean_var}
\end{figure}
Lastly, by examining the entries of \cref{tab:coststats}, it becomes evident that, true to our article's overarching problem, the RA-RRT* indeed minimizes the path length's $\cvar$, evidenced by its smaller values across the board compared to the standard RRT* algorithm. \revadd{Consequently, the expectation is also minimized (see the penultimate column of \cref{tab:coststats}), which provides evidence for the optimality of our approach}. 

\subsection{A Note On Performance at the Extremities of Uncertainty}\label{ssec:extrperf}
We conclude this section with a discussion on the planner performance at the extremities of uncertainty, i.e., for $\sigma_{C_k} = 0$ and $\sigma_{C_k} >> 0.5$. In the former case and for the same planning environment and problem setup, we expect the performance of the RA-RRT* algorithm to coincide with the baseline RRT*, since the VaR and CVaR become equal (see \cref{eq:varform,eq:cdfform,eq:cvarformck}). This is most likely to be the case, since in the noise-free case, our algorithm is essentially the $k_n$-RRT* algorithm which has been shown to have comparable computational efficiency with the RRT* \revadd{algorithm} (see \cite{karaman2011sampling}, Table 1). In the latter case, however, we anticipate a significantly-degraded performance for both algorithms, since finitude of the path segment lengths is assumed for tractability. Still, we expect a possibly less severe performance degradation for the RA-RRT* algorithm than the baseline.
\begin{table}[t]
    \centering
    \caption{\sc Cost Statistics Comparisons ($\sigma_{C_k} = 0.5$)}\label{tab:coststats}
    \begin{adjustbox}{max width=\textwidth}
    \begin{minipage}{\columnwidth}
    \renewcommand*\footnoterule{}
    \begin{savenotes}
    \begin{tabularx}{\columnwidth}{C{1.5cm}|XXXXC{1cm}|C{.3cm}}
    \toprule
      \mcrot{1}{l}{0}{Algorithm\footnote{As additional support for our results, we notice here that, consistent with the known VaR-CVaR relationship established in the literature \cite{rockafellar_optimization_2000}, for both values of $\alpha$, $\cvar[\revadd{\Lagr}] > \var[\revadd{\Lagr}]$.}} & \mcrot{1}{l}{60}{$\min{\op{VaR}_{0.1}}$} & \mcrot{1}{l}{60}{$\min{\op{VaR}_{0.9}}$} & \mcrot{1}{l}{60}{$\min{\op{CVaR}_{0.1}}$} & \mcrot{1}{l}{60}{$\min{\op{CVaR}_{0.9}}$} & \mcrot{1}{l}{60}{$\min\mathbb{E}[\Lagr]$} & \mcrot{1}{l}{0}{$\alpha$} \\
      \midrule
      \midrule
      \multirow{1}{*}{\small $\text{RRT*}$} 
                                    & 14.25 & 3.77 & 28.89 & 6.10 &	15.82 & - \\
     \midrule
     \multirow{3}{*}{\small $\text{RA-RRT*}$} 
                                    & 10.71 & 2.21 & 22.55 & 4.51 & 12.16 & 0.1 \\
                                    & 10.17 & 3.02 & 21.18 & 4.98 & 11.90 & 0.5 \\
                                    & 14.34 & 3.57 & 26.44 & 6.34 & 15.65 & 0.9 \\
     \bottomrule
    \end{tabularx}
    \end{savenotes}
    \end{minipage}
    \end{adjustbox}
\end{table}

\section{Concluding Remarks}\label{sec:conc}
In this article, we developed a probabilistically-robust sampling-based algorithm for solving SSP problems by adapting the RRT* algorithm to handle uncertainty in path-segment lengths through risk-sensitive optimization. Using formal arguments, algorithmic analyses, and results from exhaustive simulations of a grid-world path planning experiment, we demonstrated the utility of adopting risk in incremental sampling-based planning algorithms. We also presented a comprehensive computational complexity analysis of our RA-RRT* algorithm, which demonstrated that despite a slight increase in processing time, our approach maintains comparable query time and memory space complexity to the baseline RRT* algorithm, while significantly reducing planner failure rates and providing robustness to \revadd{environmental} uncertainty. In future work, we hope to adapt the foregoing ideas to the setting of stochastic environments with dynamic and possibly noisy obstacles. Other valid extensions entail developing similar ideas for non-Gaussian and non-additive noise or for the case where the CVaR cannot be computed precisely.

\bibliographystyle{ieeetr} 
\bibliography{ref}

\end{document}

%% file: macros_ieeeconf.tex
\PassOptionsToPackage{prologue,dvipsnames}{xcolor}

\expandafter\ifx\csname pdfoptionalwaysusepdfpagebox\endcsname\relax\else
\fi

\usepackage{footnote}
\usepackage[hang,flushmargin]{footmisc}

\usepackage[utf8]{inputenc}
\usepackage[T1]{fontenc}
\usepackage{subfiles}

\usepackage[dvipsnames]{xcolor}

\usepackage{microtype}
\usepackage{parskip}
\usepackage{ragged2e}
\usepackage[pdfencoding=auto, psdextra]{hyperref}
\hypersetup{
colorlinks=true,
linkcolor=arxblue,
filecolor=black,      
urlcolor=RedViolet,
citecolor = arxblue,
}
\usepackage{soul}
\usepackage{pdfpages}

\usepackage{algpseudocode}
\usepackage{listings}

\definecolor{codegreen}{rgb}{0,0.6,0}
\definecolor{codegray}{rgb}{0.5,0.5,0.5}
\definecolor{codepurple}{rgb}{0.58,0,0.82}
\definecolor{backcolour}{rgb}{1,1,1}
\definecolor{red}{rgb}{1,0,0}
\definecolor{green}{rgb}{0,1,0}
\definecolor{yellow}{rgb}{1,1,0}
\definecolor{orange}{rgb}{1,0.647,0}
\definecolor{gold}{rgb}{1,0.843,0}
\definecolor{purple}{rgb}{0.627,0.125,0.941}
\definecolor{gray}{rgb}{0.745,0.745,0.745}
\definecolor{brown}{rgb}{0.647,0.165,0.165}
\definecolor{navy}{rgb}{0,0,0.502}
\definecolor{pink}{rgb}{1,0.753,0.796}
\definecolor{seagreen}{rgb}{0.18,0.545,0.341}
\definecolor{turquoise}{rgb}{0.251,0.878,0.816}
\definecolor{violet}{rgb}{0.933,0.51,0.933}
\definecolor{darkblue}{rgb}{0,0,0.545}
\definecolor{darkcyan}{rgb}{0,0.545,0.545}
\definecolor{darkgreen}{rgb}{0,0.392,0}
\definecolor{darkmagenta}{rgb}{0.545,0,0.545}
\definecolor{darkorange}{rgb}{1,0.549,0}
\definecolor{darkred}{rgb}{0.545,0,0}
\definecolor{lightblue}{rgb}{0.678,0.847,0.902}
\definecolor{lightcyan}{rgb}{0.878,1,1}
\definecolor{lightgray}{rgb}{0.827,0.827,0.827}
\definecolor{lightgreen}{rgb}{0.565,0.933,0.565}
\definecolor{lightyellow}{rgb}{1,1,0.878}
\definecolor{black}{rgb}{0,0,0}
\definecolor{white}{rgb}{1,1,1}
\lstdefinestyle{mystyle}{
    backgroundcolor=\color{backcolour},   
    commentstyle=\color{codegreen},
    keywordstyle=\color{magenta},
    numberstyle=\tiny\color{codegray},
    stringstyle=\color{codepurple},
    basicstyle=\ttfamily\footnotesize,
    breakatwhitespace=false,         
    breaklines=true,                 
    captionpos=b,                    
    keepspaces=false,                 
    numbers=left,                    
    showspaces=false,                
    showstringspaces=false,
    showtabs=false,                  
    tabsize=1,
}
\let\labelindent\relax %
\usepackage[shortlabels]{enumitem}

\usepackage{dsfont}
\usepackage{dingbat}
\usepackage{mdframed}

\usepackage{graphicx}
\usepackage{float}
\usepackage[]{subcaption}

\newcommand{\lwdth}{2pt}
\mdfsetup {
        linewidth = \lwdth,
        innerleftmargin = 0pt,
        innerrightmargin = 7pt,
        innertopmargin = 0,
        innerbottommargin = 0,
         linecolor=Black,
        }

\usepackage{svg}
\usepackage{siunitx}
\usepackage{multirow}
\usepackage{rotating}
\usepackage{makecell}
\usepackage{booktabs}
\usepackage{tabularx}
\makeatletter
\def\thickhline{%
  \noalign{\ifnum0=`}\fi\hrule \@height \thickarrayrulewidth \futurelet
   \reserved@a\@xthickhline}
\def\@xthickhline{\ifx\reserved@a\thickhline
               \vskip\doublerulesep
               \vskip-\thickarrayrulewidth
             \fi
      \ifnum0=`{\fi}}
\makeatother

\newlength{\thickarrayrulewidth}
\usepackage{makecell} %
\usepackage{tablefootnote}
\usepackage{array}
\newcommand{\PreserveBackslash}[1]{\let\temp=\\#1\let\\=\temp}
\newcolumntype{C}[1]{>{\PreserveBackslash\centering}p{#1}}
\newcolumntype{R}[1]{>{\PreserveBackslash\raggedleft}p{#1}}
\newcolumntype{L}[1]{>{\PreserveBackslash\raggedright}p{#1}}
\usepackage{supertabular}

\definecolor{mplblue}{RGB}{31, 119, 180}
\definecolor{arxblue}{RGB}{48, 51, 154}
\definecolor{mplred}{RGB}{214, 39, 40}
\definecolor{mplgrey}{RGB}{127, 127, 127}
\definecolor{darkgray}{rgb}{0.663,0.663,0.663}

\usepackage{algorithm}
\usepackage{algorithmicx}
\usepackage[italicComments=true,commentColor=arxblue]{algpseudocodex}
\newcommand{\Input}{\textbf{Input}: }
\renewcommand{\Output}{\textbf{Output}: }
\newcommand{\Inputs}{\textbf{Inputs}: }

\newcommand{\Parameters}{\textbf{Parameters}: }

\usepackage{amsthm}
\usepackage{amsmath}
\usepackage[capitalize]{cleveref} %
\crefformat{equation}{(#2#1#3)} %
\usepackage{amssymb,amsfonts, mathtools}
\usepackage{bbm}
\usepackage{bm}
\usepackage{cuted}
\usepackage{etoolbox}
\usepackage{upgreek}
\usepackage{nicefrac} %

\makeatletter
\def\@endtheorem{\endtrivlist}%
\makeatother

\makeatletter
\newtheoremstyle{mytheorem}%
  {3pt}%
  {3pt}%
  {\itshape}%
  {}%
  {\itshape\bfseries}%
  {:}%
  {.5em}%
  {\thmname{#1}\normalfont\itshape\thmnumber{\@ifnotempty{#1}{ }#2}%
   \thmnote{ {\the\thm@notefont(\itshape #3)}}}%
\makeatother
\theoremstyle{mytheorem}

\AtBeginEnvironment{gather}{\setcounter{equation}{0}}
\newtheorem{lemma}{Lemma}
\newtheorem{problem}{Problem}

\newtheorem{proposition}{Proposition}
\newtheorem{assumption}{Assumption}

\makeatletter
\newtheoremstyle{def}%
  {3pt}%
  {3pt}%
  {}%
  {}%
  {\itshape\bfseries}%
  {:}%
  {.5em}%
  {\thmname{#1}\normalfont\itshape\thmnumber{\@ifnotempty{#1}{ }#2}%
   \thmnote{ {\the\thm@notefont(\itshape #3)}}}%
\makeatother
\theoremstyle{def}
\newtheorem{definition}{Definition}

\newcommand{\gauss}[2]{\mathcal{N}(#1, #2)}
\newcommand{\tee}[1]{\texttt{#1}}

\newcommand{\mc}[1]{\mathcal{#1}}

\newcommand{\ith}{i^{\text{th}}}
\newcommand{\kth}{k^{\text{th}}}
\newcommand{\jth}{j^{\text{th}}}

\newcommand{\x}{\mathbf{x}}
\newcommand{\p}{\mathbf{p}}

\newcommand{\cvar}{{\operatorname{CVaR}_{\alpha}}}
\newcommand{\var}{{\operatorname{VaR}_{\alpha}}}

\usepackage{upgreek}
\usepackage{wasysym}

\newcommand{\V}{\mathcal{V}}
\newcommand{\E}{\mathcal{E}}

\newcommand{\xfree}{\operatorname{\mathcal{X}_\text{free}}}

\usepackage[mode=buildnew,subpreambles=true]{standalone}
\usepackage{adjustbox}

\usepackage{nicematrix}

\crefformat{assumption}{Assumption #1}

\crefformat{equation}{(#2#1#3)} %

\crefformat{assumption}{Assumption #1}

\makeatletter
\def\@opargbegintheorem#1#2#3{\trivlist
   \item[]{\itshape #1\ #2\ (#3)}\\}
\makeatother

\newcommand{\Xobs}{{\mathcal{X}_{\text{o}}}}
\newcommand{\xgoal}{x_{\text{goal}}}
\newcommand{\Xgoal}{\mc{X}_{\text{goal}}}
\newcommand{\xinit}{x_{\text{start}}}
\newcommand{\nmax}{N_{\text{max}}} %

\newcommand{\Prob}[2][]{%
  \operatorname{Pr}\ifthenelse{\isempty{#1}}{}{\left(#2\right)}%
}

\newlist{prbdes}{itemize}{1}
\setlist[prbdes]{label={${\mathcal{P}_1}$:\hspace*{-12pt}}}
\makeatletter
\patchcmd{\enit@itemize@i}{\fi}{\fi\ifnum\pdfstrcmp{\@currenvir}{prbdes}=0 \item\fi}{}{}
\makeatother

\newcommand{\revadd}[1]{{#1}}

\newcommand{\dgridx}{{d_{x}}}
\newcommand{\dgridy}{{d_{\text{y}}}}
\newcommand{\vertmat}[2]{\begin{bmatrix}#1\\#2\end{bmatrix}}

\newcommand{\polyobs}[1]{%
  \ifthenelse{\isempty{#1}}{\mathcal{P}_o}{\mathcal{P}_{o,#1}}%
}
\newcommand{\ellipsobs}[1]{%
  \ifthenelse{\isempty{#1}}{\mathcal{E}_o}{\mathcal{E}_{o,#1}}%
}

\newcommand{\polyobsnoisek}[1]{%
  \ifthenelse{\isempty{#1}}{\delta^\mathcal{P}_{k,t}}{\delta^\mathcal{P}_{k,#1}}%
}

\newcommand{\ellipsobsnoisel}[1]{%
  \ifthenelse{\isempty{#1}}{\delta^\mathcal{E}_{l,t}}{\delta^\mathcal{E}_{l,#1}}%
}

\newcommand{\xnear}{x_{\text{near}}}

\newcommand{\xnearest}{x_{\text{nearest}}}

\newcommand{\xnew}{x_{\text{new}}}
\newcommand{\xrand}{x_{\text{rand}}}

\newcommand{\normdiff}[2]{\lvert\lvert #1 - #2\rvert\rvert}
\newcommand{\norm}[1]{\lvert\lvert #1 \rvert\rvert}

\newcommand{\setz}[2]{\mathbb{Z}_{{#1},{#2}}}
\newcommand{\setr}[2]{\mathbb{R}_{{#1},{#2}}}
\newcommand{\erfc}[1]{%
  \ifthenelse{\isempty{#1}}{\operatorname{erfc}}{\operatorname{erfc}\left(#1\right)}%
}
\newcommand{\obsappx}{\Tilde{\mc{O}}}
\newcommand{\numobsappx}{m_{\obsappx}}

\newcommand{\maxiter}{\nmax}
\newcommand{\numcostdistsamp}{n^k_c}
\newcommand{\rneighbors}{R_{m}}
\newcommand{\xmincvar}{x_{\min}}
\newcommand{\expt}{\mathbb{E}}
\newcommand{\variance}{\operatorname{Var}}

\newcommand{\scs}[1]{{\scshape{#1}}}
\newcommand{\rrad}{{R}_{\text{rb}}}

\usepackage{mathrsfs}
\newcommand{\Lagr}{L} %

\newcommand{\op}[1]{\operatorname{#1}}
\newcommand{\nrarrt}{n_{\mathrm{RA}}}
\newcommand{\ncvark}{n^k_{\mathrm{CVaR}}}

\newcommand{\mcrot}[4]{\multicolumn{#1}{#2}{\rlap{\rotatebox{#3}{#4}~}}} 

\newcommand*{\twoelementtable}[3][l]%
{%
    \begin{tabular}[t]{@{}#1@{}}%
        #2\tabularnewline
        #3%
    \end{tabular}%
}

\usepackage{silence}
\newcommand{\lmax}{\Lagr_{\mathrm{max}}}

\newcommand{\lworst}{\Lagr_{\mathrm{worst}}}

\newcommand{\Jcvar}{J_{\mathrm{CVaR}}}

\usepackage{tikz}
\usetikzlibrary{spy}
\algrenewcommand\algorithmicforall{\textbf{for each}}
\algdef{S}[FOR]{ForEach}[1]{\algorithmicforall\ #1\ \algorithmicdo}

\newcommand{\rewirerad}{\rho}
\newcommand{\maxrewirerad}{\rho_{\mathrm{max}}}

\newcolumntype{f}{>{\centering\arraybackslash}p{.2\columnwidth}}
\newcolumntype{s}{>{\centering\arraybackslash}p{.08\columnwidth}}
\newcolumntype{z}{>{\centering\arraybackslash}p{.1\columnwidth}}

\usepackage{tabstackengine}
\setstackEOL{\cr} %
\newcommand{\muell}{\mu_{\Lagr}}
\newcommand{\sigmaell}{\sigma_{\Lagr}}
\newcommand{\sigmaellsq}{{\sigma^2}_{\Lagr}}

\newcommand{\muellstar}{\mu_{\lworst^\star}}
\newcommand{\sigmaellstar}{\sigma_{\lworst^\star}}
\newcommand{\sigmaellstarsq}{{\sigma^2}_{\lworst^\star}}

\newcommand{\len}[1]{%
  \ifthenelse{\isempty{#1}}{\ell}{\ell\left(#1\right)}%
}

\usepackage{anyfontsize} %
\newcommand{\lstarworst}{\lworst^\star}